\documentclass{article}

\usepackage{natbib} 
\usepackage{microtype}
\usepackage{graphicx}
\usepackage{subcaption}
\usepackage{booktabs} 

\usepackage{tikz}
\usetikzlibrary{arrows.meta, positioning, decorations.pathreplacing, calc, fit}

\newcommand{\threshold}{\theta_0}

\newcommand{\mathsep}{,~}
\newcommand{\st}{\ensuremath{~\middle|~}}
\newcommand{\card}[1]{\left\lvert{#1}\right\rvert}
\newcommand{\absv}[1]{\card{#1}}
\newcommand{\norm}[2]{\left\lVert{#1}\right\rVert_{#2}}

\newcommand{\set}[1]{\left\lbrace #1 \right\rbrace}

\newcommand{\setR}{\mathbb R}

\newcommand{\expect}{\mathbb E}

\newcommand{\bA}{{\bf A}}
\newcommand{\Na}{N_{a}}
\newcommand{\bC}{{\bf C}}
\newcommand{\Nc}{N_{c}}
\newcommand{\bD}{{\bf D}}
\newcommand{\Sign}{\textsc{Sign}}

\newcommand{\compcost}{c_c}

\newcommand{\rtcost}{c_r}
\newcommand{\pc}{p_c}
\newcommand{\budget}{\mathbf{b}}

\newcommand{\ncomps}{N_c}

\newcommand{\nrts}{N_r}

\usepackage{hyperref}

\usepackage[preprint]{icml2026}

\usepackage{amsmath}
\usepackage{amssymb}
\usepackage{mathtools}
\usepackage{amsthm}

\usepackage{thmtools}
\usepackage{thm-restate}

\usepackage[capitalize,noabbrev]{cleveref}

\theoremstyle{plain}
\newtheorem{theorem}{Theorem}[section]
\newtheorem{proposition}[theorem]{Proposition}
\newtheorem{lemma}[theorem]{Lemma}

\theoremstyle{definition}
\newtheorem{definition}[theorem]{Definition}
\newtheorem{remark}[theorem]{Remark}

\usepackage[textsize=tiny]{todonotes}

\icmltitlerunning{Scoring from Comparisons and Ratings}

\begin{document}

\twocolumn[
  \icmltitle{The Benefits of Diversity: \\ Combining Comparisons and Ratings for Efficient Scoring}




  \icmlsetsymbol{equal}{*}

  \begin{icmlauthorlist}
    \icmlauthor{Julien Fageot}{}
    \icmlauthor{Matthias Grossglauser}{epfl}
    \icmlauthor{Lê-Nguyên Hoang}{tou}
    \icmlauthor{Matteo Tacchi-Bénard}{cnrs,tou}
    \icmlauthor{Oscar Villemaud}{epfl,tou}
  \end{icmlauthorlist}

  \icmlaffiliation{epfl}{EPFL, Lausanne, Switzerland}
  \icmlaffiliation{tou}{Tournesol association, Switzerland}
  \icmlaffiliation{cnrs}{Univ. Grenoble Alpes, CNRS, Grenoble INP, GIPSA-lab, Grenoble, France}

  \icmlcorrespondingauthor{Julien Fageot}{julien.fageot@gmail.com}
  \icmlcorrespondingauthor{Oscar Villemaud}{oscar.villemaud@epfl.ch}

  \icmlkeywords{Machine Learning, ICML}

  \vskip 0.3in
]



\printAffiliationsAndNotice{}  

\begin{abstract}
  Should humans be asked to evaluate entities individually or comparatively?
This question has been the subject of long debates.
In this work, we show that, interestingly, 
combining both forms of preference elicitation 
can outperform the focus on a single kind.
More specifically, we introduce SCoRa 
(Scoring from Comparisons and Ratings), 
a unified probabilistic model that allows to learn from both signals.
We prove that the MAP estimator of SCoRa is well-behaved.
It verifies monotonicity and robustness guarantees.
We then empirically show that SCoRa recovers accurate scores, even under model mismatch.
Most interestingly, we identify a realistic setting 
where combining comparisons and ratings outperforms using either one alone, 
and when the accurate ordering of top entities is critical. 
Given the \emph{de facto} availability of signals of multiple forms,
SCoRa additionally offers a versatile foundation for preference learning.
 
\end{abstract}

\section{Introduction}
\label{sec:introduction}
In many applications, 
like content recommendation and language model alignment,
there are ethical and performance motivations 
to fine-tune algorithms based on user preferences.
Debates are ongoing on how such preferences should be elicited,
with two leading candidates,
namely, direct item ratings and comparison-based judgments.

The former is used, for example, 
by the Internet Movie Database (IMDb)
where the items are graded with scores between 1 and 10,
while social media often leverages likes or dislikes.
Such direct rating systems are arguably easier to use, 
thereby allowing a large collection of data.
However, this approach has also been argued to provide unreliable data.
In particular, top content may have ``saturated'' maximal scores, 
thereby preventing the systems from differentiating them.
This is especially an issue when identifying an accurate ranking of the top 1\% of items 
is of high importance, e.g., for a recommender system.

Conversely, 
comparisons have been argued to better fit humans' judgment process~\cite{festinger1954theory, shah2016estimation},
have been used in psychophysics~\cite{thurstone1927law}, 
climate studies~\cite{kristof2019user}, 
ethical judgments~\cite{awad2018moral,DBLP:journals/pacmhci/LeeKKKYCSNLPP19,hoang2021tournesol,hoang2022tournesol}, 
and reinforcement learning with human feedback~\cite{DBLP:conf/nips/ChristianoLBMLA17,DBLP:conf/icml/ZhuJJ23}. 
However, they require being mindful of two items at once, 
which may be cognitively demanding, 
especially for the evaluation of texts or videos.
But this friction can also be argued to invite more thoughtful judgments.
Typically, it may be effortless to rate a climate denialist content as harmful,
and that an IPCC expert's discourse is reliable,
but more challenging (and instructive) to determine 
which of two climate change explainers should be more often recommended,
or generated by language models.
Comparisons may then be more informative once elicited, despite being more costly to elicit.

\begin{table*}
\begin{center}
\begin{tabular}{|c|p{5.5cm}|p{5.5cm}|}
\hline
\textbf{Criterion} & \textbf{Comparisons} & \textbf{Ratings} \\
\hline
\textbf{Cognitive effort} & Requires evaluating two items at once, potentially cognitively demanding & Quick and simple to provide, often requires only a single click or a number \\
\hline
\textbf{Data richness} & Provides fine-grained relative preferences between items & Captures independent judgments, often with limited resolution \\
\hline
\textbf{User engagement} & Less frequent due to higher effort, may lead to fewer data points & Easier to scale and collect at large volumes \\
\hline
\textbf{Robustness} & Less prone to biases like rating inflation or saturation & Susceptible to inconsistent scales, context effects, and biases \\
\hline
\textbf{Model interpretability} & Naturally leads to relative rankings & Enables absolute scoring, useful for threshold-based decisions \\
\hline
\textbf{Use cases} & Better for distinguishing between top-tier items or fine-tuning rankings & Better for fast evaluations and broad coverage \\
\hline
\end{tabular}
\caption{A comparison between comparisons and ratings.}
\vspace{-20pt}
\end{center}
\end{table*}

So, \emph{should practitioners elicit comparisons or ratings?}
In this paper, we suggest a third possibility: \emph{query both}.

\paragraph{Related works.}
The question of which modality to use when collecting human feedback received much interest in recent years.
In particular, the question of whether to use ratings or comparisons (or rankings) data has been the topic of several works \citep{shah2016estimation, babski2023inferring, fernandes2023bridging}.
Evidently, the debate may not be fully well-posed,
given that there may be hybrid feedback, 
such as the one proposed by \citet{xu2023perceptual}.

Most of the published findings argued that comparisons are preferable empirically \citep{shah2016estimation, babski2023inferring, fernandes2023bridging}
 and theoretically \citep{villemaud2025ranking}.
To the best of our knowledge, only \citet{wang2019your} argue the opposite.
Yet since ratings remain the default way of collecting feedback in most applications, 
and given that they are much more abundant than comparisons,
discarding them altogether is arguably a missed opportunity.

A few prior works proposed algorithms to leverage both ratings and comparisons.
In particular, \citet{ye2014active}  and \citet{perez2019pairwise} 
introduced algorithms based on noisy ratings and Bradley-Terry comparisons.
Interestingly, our SCoRa family of models essentially generalizes their solutions, 
which approximately correspond to using a Gaussian root law for ratings (up to some rating discretization)
and a binary root law for comparisons (see the root law definition below).

\paragraph{Contributions.}
First, we propose a class of models, 
which we call \emph{Scoring from Comparisons and Ratings} (SCoRa),
that leverage both signals.
Our approach is comparison-centered, 
in the sense that the ratings are themselves interpreted 
as comparisons with a fictive entity.
It leverages recent advances in comparison-based learning~\cite{fageot2023generalized,fageot2025generalizing}.
Our class of models also allow to account for metadata represented as embeddings.

Second, we prove several of desirable properties of SCoRa models,
such as yielding a convex negative log-posterior,
guaranteeing monotonicity for ratings and comparisons,
and providing a Lipschitz resilience.
Additionally, we evaluate our model, 
and show its ability to recover reasonable scores despite model mismatch.

Third, and most interestingly, 
we identify a simple and realistic settings 
where, given a fixed data elicitation budget, 
querying both ratings and comparisons improves accuracy
over collecting ratings only, or gathering comparisons only.
Essentially, this setting arises 
when the correct scoring of top entities matters most
and when active learning is used to prioritize their comparisons.

\paragraph{Structure.}
The rest of the paper is organized as follows.
Section~\ref{sec:model} introduces the SCoRa model.
In Section~\ref{sec:property}, we obtain its theoretical properties.
Section~\ref{sec:evaluation} presents our empirical evaluations. 
In Section~\ref{sec:SCoRa_value}, we identify the regimes where combining ratings and comparisons is beneficial.
Finally, Section~\ref{sec:conclusion} concludes.
Proofs are postponed to the Appendix and experiment can be accessed in Section \ref{sec:evaluation}. 

\section{The SCoRa Model}
\label{sec:model}
In many applications, entities are evaluated through several types of observations, such as \emph{pairwise comparisons}, \emph{individual ratings}, and \emph{feature embeddings}. Each source provides only partial information, but together they reveal a richer picture of the underlying latent scores.
The SCoRa model (\textbf{S}cores \textbf{Co}mparisons and \textbf{Ra}tings) offers a unified probabilistic framework for combining these heterogeneous observations into a coherent set of interpretable scores. 

\subsection{Comparisons, Ratings, and Embeddings}
\label{sec:compassessscore}

We consider a collection of $A$ entities indexed by $[A] = \{1, \dots, A\}$.  
Each entity $a \in [A]$ is represented by a feature vector $x_a \in \mathbb R^D$, 
and we write $x = (x_a)_{a \in [A]} \in \mathbb R^{D \times A}$
for the feature embedding matrix.  
A parameter vector $\beta \in \mathbb R^D$ induces the latent scores
$\theta = (\theta_a)_{a \in [A]} = x^\top \beta \in \mathbb R^A$
through $\theta_a = x_a^\top \beta$.

To learn preferences, we assume that two datasets have been collected.
First is a dataset of $\Na$ ratings $(a_n, t_n)$,
where $a_n \in [A]$ is the entity being rated
and $t_n \in \setR$ is its ratings.
We denote $\bA \triangleq \set{(n, a_n, t_n) \st n \in [\Na]}$ 
the multiset of ratings.
Second is a dataset of $\Nc$ comparisons $(b_n, c_n, r_n)$,
where $b_n, c_n \in [A]$ are the distinct entities being compared
and $r_n \in \setR$ is the comparative judgment.
We denote $\bC \triangleq \set{(n, b_n, c_n, r_n) \st n \in [\Nc]}$ 
the multiset of comparisons.
When iterating over $\bA$ or $\bC$, unless some ambiguity arises,
we will drop the dependency on the index $n$.
Finally, let $\bD \triangleq \bA \cup \bC$ be 
the entire dataset of both comparisons and ratings.

\subsection{The SCoRa Scoring Model}
\label{sec:GBTdatamodel}

The \emph{SCoRa model} provides a probabilistic model 
of how the latent feature vector $\beta \in \mathbb R^D$
generates comparisons and ratings.
This model has an additional latent variable $\theta_0 \in \mathbb R$,
which serves as a shared baseline for ratings.
More specifically, SCoRa is defined as follows.

\begin{enumerate}
    \item \textbf{Conditional independence.} All observations are conditionally independent given $(\beta, \theta_0)$.

    \item \textbf{Embedding.} All alternatives $a \in [A]$ have an associated embedding (or feature) vector $x_a \in \mathbb{R}^D$, with $D$ the feature dimension. 
    
    \item \textbf{Comparisons.}  
When comparing entities $b$ and $c \in [A]$, 
the random comparative judgment $r$ is assumed to 
follow a generalized Bradley--Terry (GBT) model~\cite{fageot2023generalized,fageot2025generalizing}
with root law $f$ 
(with finite exponential moments),
given entity scores $\theta_b \triangleq x_b^\top \beta$
and $\theta_c \triangleq x_c^\top \beta$, i.e.
\begin{equation}
\label{eq:lawComp}
    p(r \mid \beta)
    \propto
    f (r)\,\exp\!\big(r\, x_{bc}^\top \beta\big),
\end{equation}
with $x_{bc} \triangleq x_b - x_c \in \setR^D$ the embedding difference.

    \item \textbf{Ratings.} For the ratings of an entity $a$, 
the outcome is modeled as a GBT comparison 
between the implicit score $\theta_a = x_a^\top \beta$ of $a$ 
and the latent threshold $\theta_0 \in \setR$, 
with root law $g$ (also with finite exponential moments):
\begin{equation}
\label{eq:lawAssess}
    p(t \mid \beta, \theta_0)
    \propto
    g(t)\,\exp\!\big(t(x_c^\top \beta - \theta_0)\big).
\end{equation}

\item \textbf{Prior.}
Independent Gaussian priors are placed on $\beta$ with covariance $\Sigma_\beta = \sigma_\beta^2 I$ and on $\theta_0$ 
with variances and $\sigma_0^2$.
\end{enumerate}  

\begin{definition}
\label{def:GBTassessmodel}
A probabilistic model satisfying Assumptions~1--4 is called a \emph{SCoRa model} 
with embedding matrix $x$, comparison and rating root-laws $f$ and $g$, 
and prior variances $\sigma_\beta^2$ and $\sigma_0^2$.
\end{definition}

\paragraph{Threshold interpretation.}
The threshold $\theta_0$ provides a global reference for ratings, aligning them to the same latent scale as comparisons.  
It acts as the neutral level at which entities are judged favorably ($t>0$) or unfavorably ($t<0$).
Crucially, it is learned rather than fixed.

\paragraph{Embedding and one-hot encoding.}
The embedding matrix \( x \in \mathbb{R}^{D \times A} \) connects alternatives to feature domains. The simplest case, \( x = I_A \) (where \( D = A \)), corresponds to a trivial embedding that merely enumerates the alternatives. A more structured approach is the \emph{one-hot encoding}, where each alternative \( x_a \) is represented by a binary vector with a single \( 1 \) indicating its class membership (among \( D \) possible classes). Without loss of generality, up to a permutation of columns, the embedding matrix takes the block-diagonal form:

\[
x =
\begin{bmatrix}
\mathbf{1}_{n_1}^\top & \mathbf{0}_{n_2}^\top & \cdots & \mathbf{0}_{n_D}^\top \\
\mathbf{0}_{n_1}^\top & \mathbf{1}_{n_2}^\top & \cdots & \mathbf{0}_{n_D}^\top \\
\vdots & \vdots & \ddots & \vdots \\
\mathbf{0}_{n_1}^\top & \mathbf{0}_{n_2}^\top & \cdots & \mathbf{1}_{n_D}^\top
\end{bmatrix},
\]

where \( \mathbf{1}_n \) denotes an \( n \)-dimensional vector of ones, \( n_d \) is the number of alternatives in class \( d \), and \( A = \sum_{d=1}^D n_d \) is the total number of alternatives. This one-hot encoding structure will be further analyzed in Section~\ref{subsec:acase}.

\subsection{MAP estimation of SCoRa}
\label{sec:MAPesti}

We now describe Maximum A Posteriori (MAP) estimator 
of the feature vector $\beta \in \mathbb R^D$ 
and the rating threshold $\theta_0 \in \mathbb R$ in the SCoRa model, 
given the dataset $\mathbf D_N$, the embedding matrix $x$, and prior variances $\sigma_\beta^2$ and $\sigma_0^2$.
Following~\cite{fageot2025generalizing}, this is best done by
introducing the cumulant generating function
\begin{equation}
\Phi_f(\theta)
=
\log \int_{\mathbb R} \exp(\theta r)\, dF(r),
\end{equation}
for a root law $f$ with cumulative distribution function $F$.  

The MAP is obtained by maximizing the posterior,
which can equivalently be done by minimizing the loss function defined as
the \emph{negative log-posterior}.
This loss is given by (up to an additive constant)
\begin{equation}
\label{eq:Lx}
    \mathcal{L}(\beta, \theta_0 \mid \mathbf{D})
    = -\log p(\beta, \theta_0 \mid \mathbf{D}) + cst.
\end{equation}

\begin{proposition}
The loss admits the explicit form
\begin{align}\label{eq:Lxcompute}
    &\mathcal{L}(\beta, \theta_0 \mid \mathbf{D}) 
    = \frac{\norm{\beta}{2}^2}{2 \sigma_\beta^2} 
        + \frac{\theta_0^2}{2\sigma_0^2} \notag \\
    &+ \sum_{(b, c, r) \in \bC} \left[
        \Phi_{bc}(x_{bc}^\top \beta) - r\, x_{bc}^\top \beta 
    \right] \notag \\
    &+ \sum_{(a, t) \in \bA} \left[ 
        \Phi_{a} (x_{c}^\top \beta - \theta_0)
          - t (x_{c}^\top \beta - \theta_0) 
    \right].
\end{align}
Moreover, $\mathcal{L}$ is $\sigma_{\max}^{-2}$-strongly convex, 
where $\sigma_{\max}^2 = \max \set{\sigma_\beta^2, \sigma_0^2}$. 
In particular, there exists a unique MAP.
\end{proposition}

\begin{proof}
We apply Bayes rule, while ignoring the denominator (which is constant in the latent variables).
The prior is then turned in the first two terms,
while the two other terms correspond to the likelihood of the observed data.
The derivation follows the same footsteps as~\cite{fageot2025generalizing}.
It is known that the cumulant generating functions are convex~\cite{dembo2009large}.
The strong convexity follows from that of the prior terms.
Strong convexity then implies the existence and uniqueness of the MAP.
\end{proof}

\begin{definition}
\label{def:MAP}
The \emph{MAP estimator} of the SCoRa model is
\begin{equation}
    (\beta^*, \theta_0^*)
    = 
    \underset{(\beta, \theta_0) \in \mathbb{R}^D \times \mathbb{R}}{\arg\min}
    \mathcal{L}(\beta, \theta_0 \mid \mathbf{D}).
\end{equation}
The corresponding score estimate is $\theta^* = x^\top \beta^*$.
\end{definition}

\subsection{Flexible Generalized Bradley-Terry Model}
\label{sec:flexible}
To facilitate the proof of the properties of SCoRa,
we introduce an even more general class of preference learning models,
which we call the flexible GBT models.
This class extends the classical GBT frameworks \cite{fageot2023generalized, fageot2025generalizing} in two directions:  
(i) each comparison may be governed by its own root law, and  
(ii) the prior on $\beta$ may be any Gaussian distribution.
Moreover, the rating threshold $\theta_0$ from the SCoRa model 
is absorbed into the embedding, 
so that all terms share a homogeneous structure.

\begin{definition}[Flexible linear GBT model]
\label{def:flexibleGBT}
A \emph{flexible linear GBT model}
is a comparison-only model with latent vector $\beta \in \setR^D$,
where each comparison $(n, a_n, b_n)$ is also labeled with a root law $f_n$.
All observed data are conditionally independent on $\beta$,
and the comparison $r_n$ is assumed to follow a GBT distribution of root law $f_n$, i.e.
\begin{equation}
    p(r_n | \beta) \propto f_n (r_n) \exp (r_n x_{a_n b_n}^\top \beta).
\end{equation}
\end{definition}

Assuming that we have a Gaussian prior $\mathcal N(0, \Sigma_\beta)$ on $\beta$, 
the negative log-posterior given a multiset 
$\bD \triangleq \set{(n, a_n, b_n, r_n, f_n) \st n \in N}$ is then given by
\begin{align*}
    \mathcal L(\beta | \bD)
    = \frac{\beta^\top \Sigma_\beta^{-1} \beta}{2}
        + \sum_{(a, b, r, f) \in \bD} \left[ 
            \Phi_f( x_{ab}^\top \beta ) - r x_{ab}^\top \beta
        \right].
\end{align*}

We show that flexible linear GBT generalizes SCoRa models.

\begin{theorem}
\label{theo:SACEisFlex}
Any SCoRa model can be represented as a flexible GBT model after an appropriate augmentation of the embedding space.
\end{theorem}

\begin{proof}
    See Appendix~\ref{app:SCoRa=flex}.
\end{proof}

Thus SCoRa is fully contained in the flexible GBT framework.
This is particularly useful for theoretical analysis.
Moreover, the gained flexibility allows further combinations of various signals.
Typically, one could now combine multiple forms of comparisons,
such as binary judgments (like in Bradley-Terry) 
and more quantified judgments.

\begin{remark}
The \emph{flexible SCoRa model} generalizes the standard SCoRa framework by allowing distinct root laws $f_{bc}$ for comparisons and $g_a$ for ratings, as well as general covariance structure. As in Theorem~\ref{theo:SACEisFlex}, any such model can be recast as a flexible GBT model via embedding augmentation, preserving the theoretical guarantees of this paper (e.g., monotonicity and resilience) under appropriate adaptations.
\end{remark}

In the sequel and for ease of reading, we present results in the native SCoRa notation, where the role of the single threshold parameter~$\theta_0$ remains explicit.

\section{Monotonicity and Resilience of SCoRa}
\label{sec:property}
This section presents the main theoretical properties of the SCoRa model, 
with detailed proofs deferred to Appendix~\ref{app:theo}.  
Relying on Theorem~\ref{theo:SACEisFlex}, 
which embeds SCoRa into the flexible GBT framework, 
we derive several structural results describing its behavior 
under both comparisons and ratings.  
In particular, we study the model’s \emph{monotonicity} 
with respect to latent scores and its \emph{resilience} to noisy or unbalanced data.

\subsection{Basic Properties}
\label{subsec:basics}

We begin with the simple with no embedding, 
which amounts to assuming $A = D$ and $x = I_A$ (the identity matrix).  
In this case, the MAP estimate of the threshold admits a simple closed-form expression.

\begin{proposition}
\label{prop:MAPthreshold}
In the SCoRa model without embeddings,
\begin{equation}
\theta_0^* = - \frac{\sigma_0^2}{\sigma_\beta^2} \sum_{a \in [A]} \theta_a^*.
\end{equation}
where $(\theta_a^*,\theta_0^*)$ denotes the MAP estimator of Definition~\ref{def:MAP}.
\end{proposition}

\begin{proof}
    Likelihood terms only feature score differences,
    and are thus invariant under the same translation for all scores.
    Minimizing the quadratic negative log-prior 
    over this translation yields the proposition.
    See Appendix~\ref{proof:propmean}.
\end{proof}

In the general embedded case, $\theta_0^*$ interacts with the coordinates of $\beta^*$ via the linear map $\theta = x^\top \beta$, so no simple closed-form expression is available.

\begin{proposition}[First two conditional moments]
\label{prop:moments_embedding}
For comparisons and ratings, conditioned on $\beta$ (and on $\theta_0$ for ratings),
\begin{align}
    \mathbb{E}[r_n \mid \beta] &= \Phi_{b_n c_n}'(x_{b_n c_n}^\top \beta),
    \\
    \mathbb{V}[r_n \mid \beta] &= \Phi_{b_n c_n}''(x_{b_n c_n}^\top \beta),
    \label{eq:momentr_embed} \\
    \mathbb{E}[t_m \mid \beta,\theta_0] &= \Phi_{a_m}'(x_{a_m}^\top \beta - \theta_0), \\
    \mathbb{V}[t_m \mid \beta,\theta_0] &= \Phi_{a_m}''(x_{a_m}^\top \beta - \theta_0).
    \label{eq:momentt_embed}
\end{align}
\end{proposition}

\begin{proof}
    See Appendix~\ref{app:moments_flexGBT}.
\end{proof}

Thus, the embedding $x$ maps $\beta$ into latent scores, 
which are transformed by $\Phi'$ and $\Phi''$ 
into conditional expectations and variances.  
This recovers the GBT moment structure while incorporating ratings.  

\subsection{Monotonicity and Update Properties}

We now formalize monotonicity and the effect of adding new data.
Our definitions are adapted from~\cite{fageot2023generalized, DBLP:journals/corr/abs-2506-08998}.

\begin{definition}[Pairwise monotonicity]
A preference learning model $\bD \mapsto \theta^*(\bD)$ is pairwise monotone 
if, for all datasets $\bD$ and $\bD'$, both of the following hold:
\begin{itemize}
  \item if $\bD$ differs from $\bD'$ only on the $n$-th comparison, 
  with $a_n = a_n'$, $b_n = b_n'$ and $r_n \ge r_n'$, 
  then $\theta_{b_n c_n}^*(\bD) \ge \theta_{b_n c_n}^*(\bD')$.
  \item if $\bD$ differs from $\bD'$ only on the $n$-th rating, 
  with $a_n = a_n'$ and $t_n \ge t_n'$, 
  then $\theta_{a_n 0}^*(\bD) \ge \theta_{a_n 0}^*(\bD')$.
\end{itemize}
\end{definition}

\begin{proposition}[Pairwise monotonicity]
\label{prop:monotonicity_SCoRa_embed}
The SCoRa MAP estimator is pairwise monotone.
\end{proposition}

\begin{proof}
    The proof leverages the flexible GBT formulation.
    See Appendix~\ref{app:pairwise_monotone}.
\end{proof}

Let $\Sign$ be the sign function with $\Sign(0)=0$.

\begin{proposition}[Directional effect of a new comparison]
\label{prop:direction_comparison}
Consider any dataset $\bD$, 
and denote $\bD'$ the dataset obtained by adding the comparison $(b, c, r)$ to $\bD$.
Then
\begin{align*}
\Sign \left( \theta_{bc}^*(\bD') - \theta_{bc}^*(\bD) \right)
=
\Sign \left( r - \Phi_f' \left(\theta_{bc}^*(\bD)\right) \right).
\end{align*}
Similarly, denote $\bD'$ the dataset by adding rating $(a, t)$ to $\bD$. 
Then
\begin{align*}
\Sign \left( \theta_{a0}^*(\bD') - \theta_{a0}^*(\bD) \right)
=
\Sign \left( t - \Phi_g' \left(\theta_{a0}^*(\bD)\right) \right).
\end{align*}
\end{proposition}

\begin{proof}
    See Appendix \ref{app:update_flexGBT}.
\end{proof}

Thus the expected value of an observation is the unique fixed point 
at which the MAP score remains unchanged; 
smaller (resp.\ larger) observations decrease 
(resp.\ increase) the corresponding score difference.

\subsection{Resilience Properties}

Define an elementary edit as 
the addition or the removal of a single rating or of a single comparison.
We define $\Delta(\bD, \bD')$ as the minimal number of elementary edits 
(addition, removal or modification of a rating/comparison)
necessary to move from dataset $\bD$ to dataset $\bD'$.
A scoring method is \emph{Lipschitz-resilient} 
if the impact on learned scores of the edits is bounded
by the number of elementary edits, up to a multiplicative constant. 
We then have the following resilience guarantee.

\begin{proposition}[Lipschitz resilience]
\label{prop:resilience_SCoRa}
For all datasets $\bD$ and $\bD'$, 
the SCoRa MAP estimator \((\beta^*, \theta_0^*)\) satisfies
\begin{align} \label{eq:thetatheta0Lipschitz}
  \|\theta^*(\bD) - \theta^*(\bD')\|_2 
  &\leq L \|x \|_2 \, \Delta\big(\bD, \bD'\big), \\
  |\theta_0^*(\bD) - \theta_0^*(\bD')| 
  &\leq L \, \Delta\big(\bD, \bD'\big),  
\end{align}
with
$L \triangleq 4 \sigma_{\max}^2 \max \set{R_{\max}, T_{\max}} \max \set{ 1,  \|{x}\|_{2,\infty} }$,
where $R_{max}$ and $T_{max}$ are the maximal comparison and rating values 
(which may be infinite)
and where \(\|x\|_{2,\infty} := \max_a \|x_a\|_2\).
\end{proposition}

\begin{proof}
    See Appendix~\ref{app:resilience_flexGBT_pairs}.
\end{proof}

The Lipschitz constants $\|x\|_2 L$ and $L$ quantify the worst-case sensitivity of the estimator to data edits. 
Their dependence on the variance parameters $\sigma_\beta^2$ and $\sigma_0^2$ 
reflects the strength of regularization: 
small variances induce strong shrinkage 
toward the prior which \emph{improves} resilience,
at the cost of reduced data fitting.

\section{Empirical Convergence}
\label{sec:evaluation}
In this section, we present some empirical evaluations of the SCoRa-MAP estimators on synthetic data,  with a known ground truth.
All experiments were run several times (the number depends on the plot) 
and the average is plotted along with the 95\% confidence intervals
\footnote{For reproducibility, our code is available at \url{https://anonymous.4open.science/r/scora_icml_2026}}. 

\subsection{Generative Models}

\paragraph{GBT parameters}
The parameters of our generative model 
 are the number $A$ of entities and the ground truth parameters of the GBT model $(f^\dagger, g^\dagger)$.
We restrict our experiments to $k$-ary and uniform root laws for generation, so $f^\dagger$ and $g^\dagger$ are controlled by $k_c$ and $k_r$, with convention $k=\infty$ for (continuous) uniform.
We parametrize the root laws used by SCoRa by $(f, g)$.
Unless specified otherwise, we use the true root laws, i.e. $(f,g) = (f^\dagger, g^\dagger)$, thus controlled by $k_r$ and $k_c$.

\paragraph{Modeling relative elicitation costs.}
We assume that the user is given a \textit{budget} $\budget$ that they can spend on either ratings or comparisons, or a combination of the two. 
This budget represents the effort or the time that the user is willing to spend on answering queries. 
In order to model the cognitive costs of each tasks mentioned Section \ref{sec:introduction} we use parameters $\compcost$ and $\rtcost$, which respectively control the cost in budget units of one comparison and one rating.
To decide the budget allocation, we use the parameter $\pc$, which controls the proportion of the budget allocated to comparisons, the rest being allocated to ratings.

    Our parameters are thus a tuple $(A,  k_c, k_r, \budget, \pc, \compcost, \rtcost)$.
    Given them, the data are generated as follows.
    
    \begin{itemize}
        \item We generate (hidden) ground truth $\beta^\dagger$ and $\threshold^\dagger$ independently from iid Gaussian zero-mean priors with variance $1$.
        We then compute the ground truth scores $\theta^\dagger=x^\top \beta^\dagger$.
        \item We split the budget and compute the associated number of comparisons and ratings $\ncomps$ and $\nrts$:    $$\ncomps=\left\lfloor \frac{\pc \budget}{\compcost} \right\rfloor \quad \text{and} \quad  
        \nrts= \left\lfloor \frac{(1-\pc) \budget}{\rtcost} \right\rfloor.$$
        \item We generate a random set $\mathcal{A}$ of $\nrts$ elements of $[A]$ to be rated.
        These elements are selected uniformly at random among all alternatives, with duplicates allowed.
        \item We generate a random set $\mathcal{C}$ of $\ncomps$ pairs of (different) elements of $[A]$ to be compared.
        These pairs are selected uniformly at random among all pairs, with duplicates allowed.
        \item We generate $r$ and $t$ independently according to \eqref{eq:lawComp} and \eqref{eq:lawAssess} (where the root laws are controlled by $k_c$ and $k_r$) to construct our dataset $\bD$. 
    \end{itemize}

    The data observed by the learning algorithm is thus a tuple $(A, k_r, k_c, \bD)$.
    We then run our SCoRa-MAP estimator, 
    with $k_r$-ary as the $g$-root law, $k_c$-ary as the $f$-root law, and with the correct prior variances $1$, to obtain an estimate $\hat{\theta} = \theta^*(\bD)$ of $\theta^\dagger$,
    using LBFGS~\cite{DBLP:journals/mp/LiuN89} as the optimizer.

\paragraph{Accuracy metric}
We compute the correlation between the estimator and the ground truth $\textsc{Corr}(\hat{\theta}, \theta^\dagger)$.

\paragraph{Parameter choice}
For all our experiments, we keep fixed the values $A=100$ 
and we vary the parameters $\budget, \pc, \compcost, \rtcost, k_r, k_c$. \\
In this section, we set the embedding matrix $x$ to be identity.
We use $k_r=k_c=\infty$, i.e., we set the root laws to be uniform.
We set $\compcost=\rtcost=1$
and $\pc \in \{0, 0.5, 1\}$, which 
corresponds to the three different lines on the plots of Figure \ref{fig:convergence}.
On the right plot of Figure \ref{fig:convergence}, we also test the resilience of SCoRa to model mismatch by setting $(f,g)$ to be gaussian root laws despite $(f^\dagger, g^\dagger)$ being uniform laws.
Finally, for both plots, we vary the budget $\budget$ from $10$ to $10^5$. 

\begin{figure*}[ht]
\begin{center}
\begin{subfigure}{0.45\linewidth}
    \includegraphics[width=\linewidth]{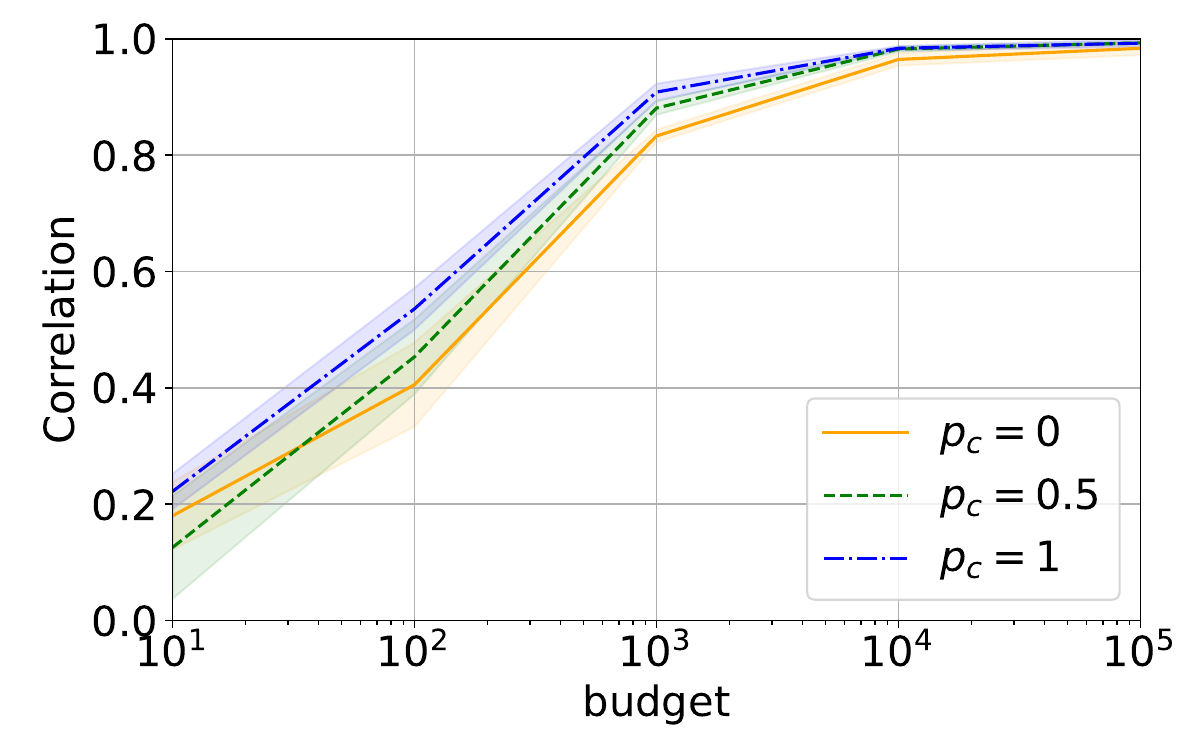}   
    \caption{Without model mismatch}
    \label{fig:convergence_match}
\end{subfigure}
\begin{subfigure}{0.45\textwidth}
    \includegraphics[width=\linewidth]{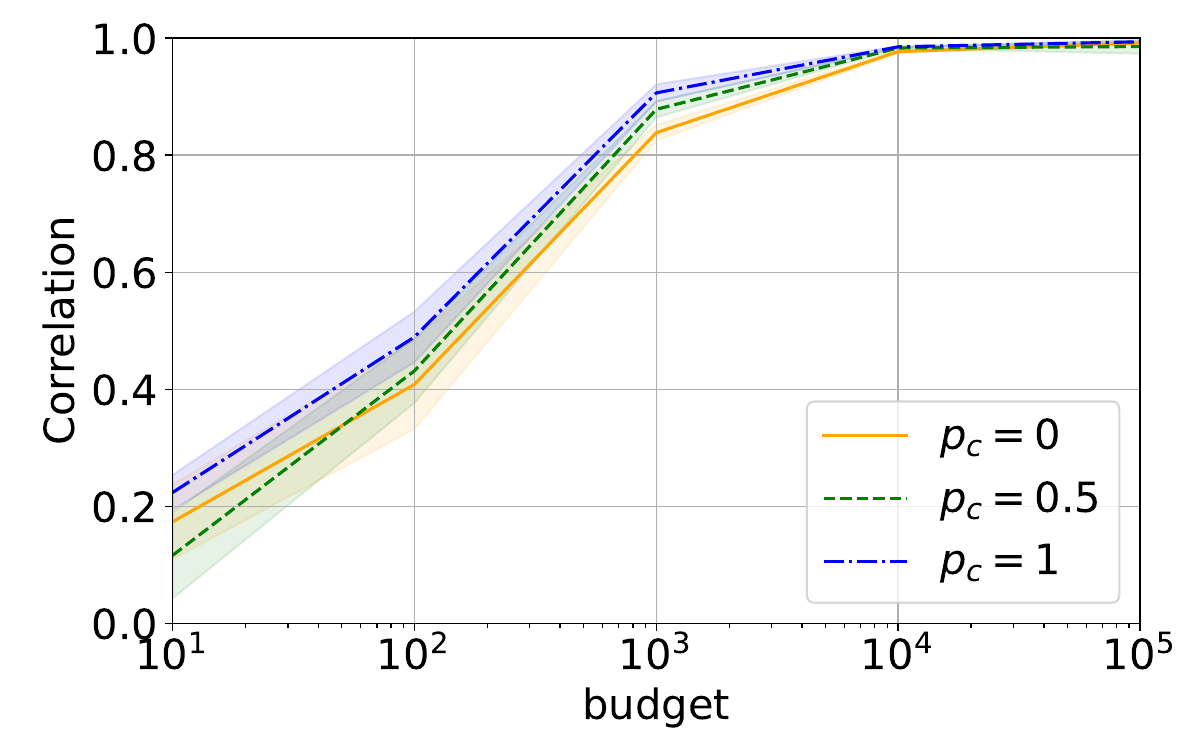}
\caption{With model mismatch}
    \label{fig:convergence_mismatch}
\end{subfigure}
    \caption{Convergence of the recovered scores irrespective of the mix, even under model mismatch. 
    We see that the correlation goes to one as the  the budget increases, whether we use only comparisons, only ratings, or a mix of the two.
    Uniform root laws are used to generate the data of both plots. 
    On the left, uniform root laws are also used for the inference, whereas gaussian root laws are used for inference on the right.
    We use parameters $k_r=k_c=\infty$, $\compcost=\rtcost=1$ and $\pc \in \{0, 0.5, 1\}$.
    } 
    \label{fig:convergence}
\end{center}
\end{figure*}

        \subsection{Results}

Figure~\ref{fig:convergence} plots the quality of the scores recovery depending on the budget, for different allocations of the budget.
Clearly, and as should be expected, increasing the number of comparisons and ratings improves the accuracy, irrespective of the mix.
We observe that whether we have only comparisons, only ratings, or a mix of the two, we asymptotically recover the scores perfectly as the budget goes to infinity.
The three curves are closely aligned, with a slight edge for pure comparisons ($p_c=1$). 
The recovery of the scores is only very slightly affected by the usage of gaussian root laws instead of uniform ones, as evidenced by the high similarity between the two plots of Figure \ref{fig:convergence}.
In the next section, we show cases where combining comparisons and ratings yields superior performance.

\section{Why Use Comparisons and Ratings?}
\label{sec:SCoRa_value}

There is a large literature that aims to compare ratings and comparisons \cite{ shah2016estimation, wang2018your}.
In this section,  we exhibit a setting 
where combining both ratings and comparisons 
outperforms both rating-only and comparison-only elicitation.

\subsection{A Case where Combining Both is Empirically Best}
    \label{subsec:acase}

From the previous simulations, 
it might seem that there is no gain in combining both ratings and comparisons.
In this section, we look for a setting, which is somewhat realistic, and where combining both ratings and comparisons is actually best.
This setting is derived from two key assumptions, namely top-entity-based accuracy metric and active learning.

\paragraph{Metric: Top entity scoring matters most.}
First, as often done in the context of recommendation systems, 
we make the hypothesis that scoring correctly the high scores is 
a lot more important than scoring correctly the low scores. 
This is typically because the system's behavior mostly depends on how top entities are recommended.
    To account for this, we introduce an accuracy measure
    which weighs more heavily top entity errors.
    More precisely, we consider a weighted correlation metric.
    \begin{equation}
    \label{eq:correxp}
        \textsc{CorrExp}(\hat{\theta}, \theta^\dagger) 
        = \frac{\sum_i w_i \hat{\theta}_i \theta^\dagger_i}{\sqrt{\sum_i w_i \hat{\theta}^2_i }\sqrt{\sum_i w_i (\theta^\dagger_i)^2}}
    \end{equation}
    where $w_i = \exp(\theta^\dagger_i)$ gives an exponentially large weight to high ground-truth scores.
    Note that this measure remains invariant to the rescaling of the learned scores. 

\paragraph{Active learning.} 
We assume that the ratings are performed first, and then used to make a first estimate of the scores in order to prioritize the comparison of well-scored alternatives for the comparisons.
In practice, it means that we run SCoRa twice.
We first run it with only the ratings, which returns some scores estimates $\Tilde{\theta}$.
We then  use $\Tilde{\theta}$ to bias the selection of the pairs compared.
More precisely, we generate a random set $\mathcal{C}$ of $\ncomps$ pairs of (different) elements of $[A]$ selected independently at random,  where the probability of a pair $a,b$ to be selected is proportional to $\exp(\Tilde{\theta}_a +\Tilde{\theta}_b)$.

\paragraph{Parameter choice.}

As in previous section, the parameters of our experiments are a tuple $(A,  k_c, k_r, \budget, \pc, \compcost, \rtcost)$.
As before, we fix the number of alternatives to $A=100$.
In this section, the coordinates of $\beta$ are generated using a Cauchy distribution of mode $0$ and scale $1$, in order to have a heavy-tailed distribution of the scores.  \\
In Figure \ref{fig:sweet-uniform} we set $\compcost=3$, $\rtcost=1$, $k_r=k_c=\infty$ and we plot three different lines for $\budget\in \{500, 1000, 1500\}$. \\
In Figure \ref{fig:sweet-binary} we set $\compcost=8$, $\rtcost=1$, $k_r=k_c=2$ and we plot three different lines for $\budget\in \{5000, 10000, 20000\}$. \\
Throughout this section, we plot the weighted correlation (see \eqref{eq:correxp}) as a function of the fraction $\pc$ of the budget allocated to comparisons.
An additional setting, in which comparisons are used for the first phase of active learning, is available in Appendix \ref{app:experiments}.

\paragraph{Adding embeddings.}
For our last experiment (Figure \ref{fig:sweet-one-hot}), we add simple embeddings using a one-hot encoding (see Section \ref{sec:GBTdatamodel}).
Each alternative is associated with a one-hot encoded vector (one and four zeros), simulating a known partition of the items into five clusters.
Items belonging to the same cluster share the same vector, which is orthogonal to those of the other clusters.
The final embedding matrix $x$, of shape $A \times (A + 5)$, is obtained by concatenating the identity matrix of size $A$ with the $A \times 5$ matrix of one-hot encoded vectors.
Each item is independently assigned to one of the five clusters with probability $0.2$.
For this experiment, we set $\compcost=8$, $\rtcost=1$, $k_r=k_c=2$, and plot three curves corresponding to $\budget\in \{10000, 20000, 100000\}$.

\subsection{Results}
The results are shown in Figures \ref{fig:sweet-uniform}, \ref{fig:sweet-binary}, and \ref{fig:sweet-one-hot}.
For a fixed budget, the best performance is consistently achieved by allocating it to a mix of comparisons and ratings, with the largest gains observed at higher budgets.
In particular, in Figure \ref{fig:sweet-one-hot}, for $\budget=10^5$, the weighted correlation increases from around $0.6$ when using only ratings or only comparisons to over $0.9$ when the budget is split.
We also observe a long plateau from $\pc=0.1$ to $\pc=0.9$, where the weighted correlation remains close to $0.9$.
As long as at least $10\%$ of the budget is allocated to each modality, performance is essentially unchanged.
This suggests that when relying on a single modality, collecting even a small number of queries from the other can lead to a substantial improvement.

Finally, comparing Figures \ref{fig:sweet-uniform} and \ref{fig:sweet-binary}, we observe that similar performance can be achieved with a significantly lower budget when using uniform rather than binary input.
This is consistent with the observations of \citet{fageot2023generalized}, although uniform input may be more sensitive to noise in practical settings.

\begin{figure}[ht]
\centering
  \includegraphics[width=\linewidth]{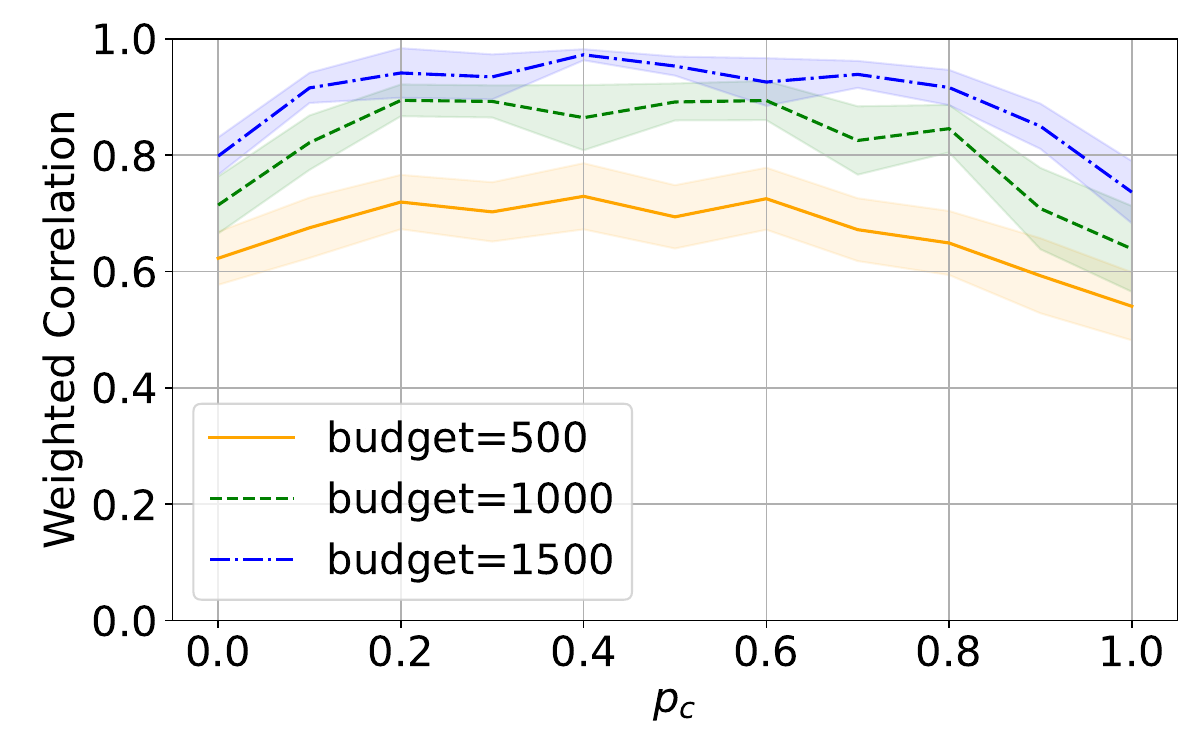}
\caption{Weighted Correlation for $\compcost=3$, $\rtcost=1$, $k_r=k_c=\infty$ and  $\budget\in \{500, 1000, 1500\}$. }
\label{fig:sweet-uniform}
\end{figure}

\begin{figure}[ht]
\centering
  \includegraphics[width=\linewidth]{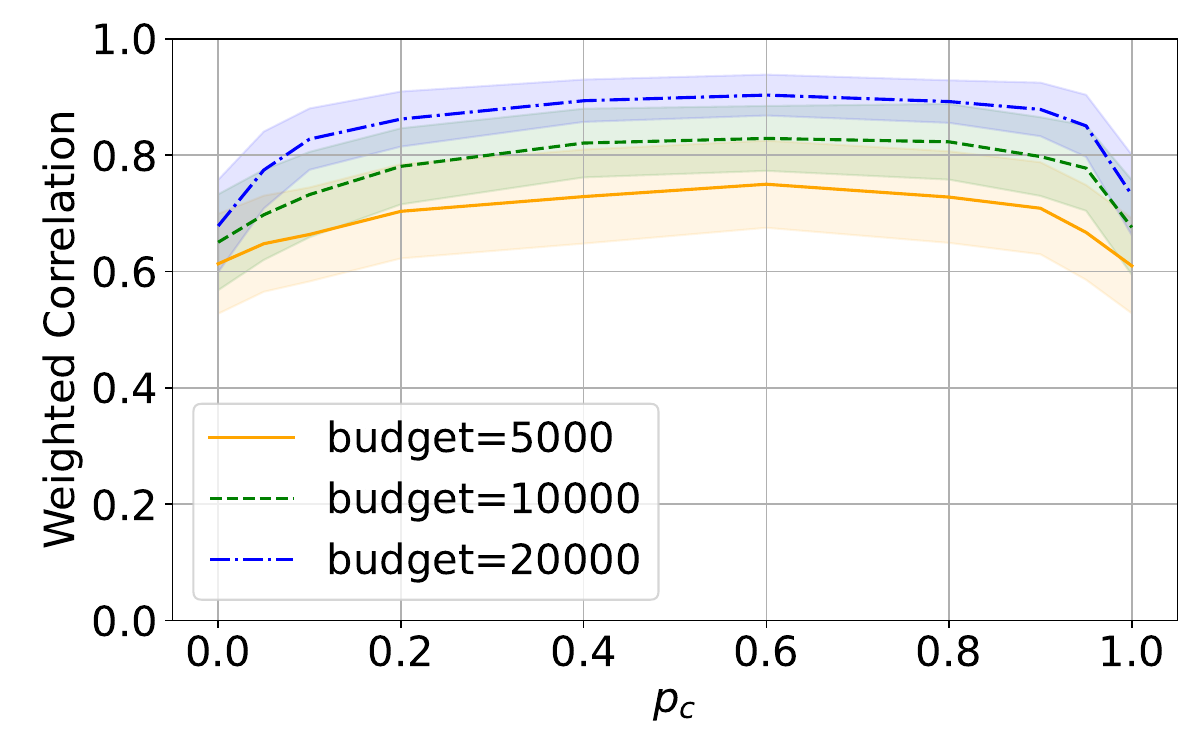}
\caption{ Weighted Correlation for $\compcost=8$, $\rtcost=1$, $k_r=k_c=2$ and $\budget\in \{5000, 10000, 20000\}$.}
\label{fig:sweet-binary}
\end{figure}

\begin{figure}[ht]
\centering
\centerline{\includegraphics[width= \linewidth]{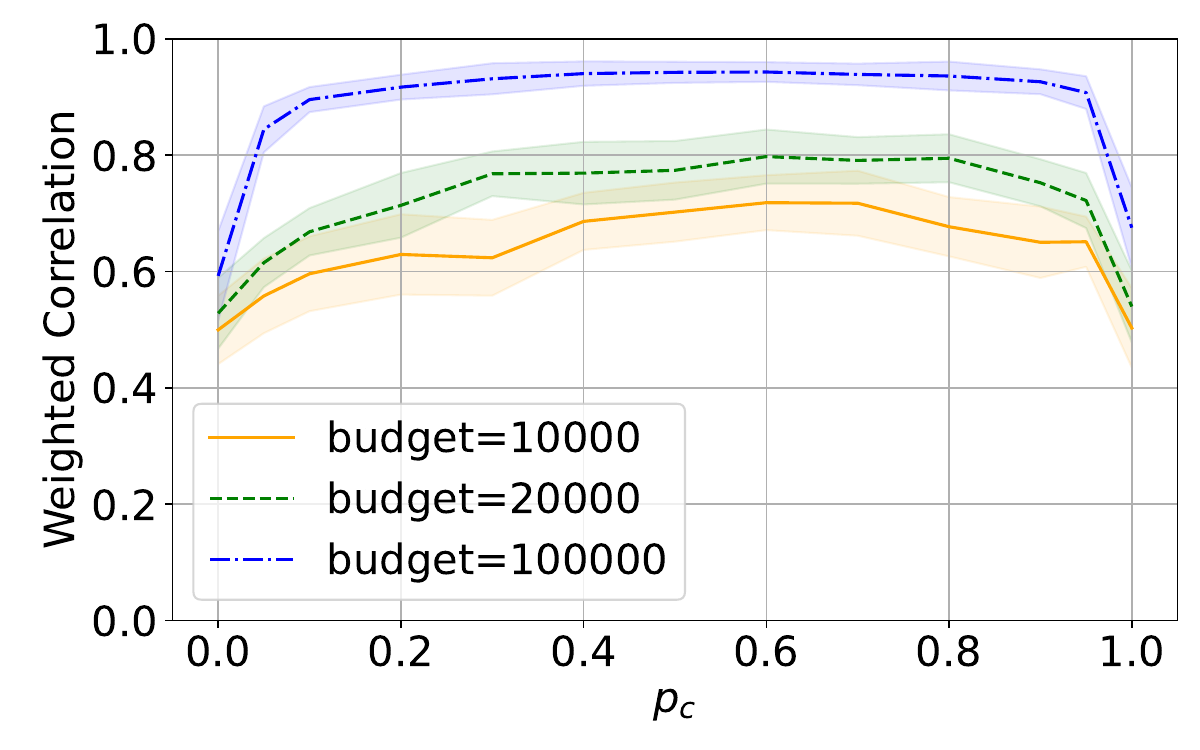}}
\caption{Weighted Correlation for $c_c=8$, $c_r=1$, $k_c=k_r=2$. 
and $\budget \in \{10^4, 2.10^4, 10^5\}$.
We use one-hot-encoded embeddings.}
\label{fig:sweet-one-hot}
\end{figure}

\vspace{-4pt}
\subsection{Intuition}

\vspace{-3pt}
There are several reasons why combining comparisons and ratings is valuable. 
Some are covered by our experiments but not all.
Let us discuss some of them informally.

\vspace{-10pt}
\paragraph{Results of the experiments.}
Let us provide an intuition for why our previous setting favored a combination of ratings and comparisons.
 As ratings are less costly, they quickly provide a lot of rough information.
However, they fail to robustly distinguish the top entities, because they all have the same reference point $\theta_0$.
But in settings where the ranking of top entities is especially important,
this deeply hurts the value of the learned scores.
This is where the costly comparisons between top entities become valuable.

\paragraph{Highly saturated ratings.}
Evidently, if the discretization of ratings is small (e.g. only like or dislike),
then ratings will fail to provide accurate scoring for top entities.
This  still be the case when ratings are more granular,
if users overuse extreme ratings.
In our experiments, this limitation is overcome by the possibility of having each entity rated arbitrarily many times. 
However, in practice it seems unrealistic to collect several independent ratings of the same item by the same user.

\paragraph{Different users have different relative comparing cost.}
This alone could justify having a system that elicits both ratings and comparisons.
In fact, even a single user could feel different elicitation costs,
depending on their fatigue or cognitive alertness.

\vspace{-2pt}
\paragraph{Heavy tail ground-truth score distribution.} 
Top entities may still have large differences, 
as was observed by~\cite{hoang2022tournesol} in the context of video recommendation.
Typically, the very top scientific papers may still be significantly more worth reading
than other top scientific papers.
The normal distribution of scores fails to capture this intuition.
At least for some applications, it can be more relevant to consider heavy-tailed distributions of ground-truth scores,
e.g. drawn from a Cauchy distribution, as in Section \ref{sec:SCoRa_value}.

\vspace{-2pt}
\paragraph{Engaging with a content is costly.}
When evaluating entities like long videos, generated texts, or books,
most of the cost lies in the analysis of the entities.
In particular, once analyzed, the cost of rating the entity is low,
while the cost of comparing it to already analyzed entities can be reasonable,
compared to the cost of rating an unanalyzed entity.
Combining ratings and comparisons seems more useful in such a setting.
\vspace{-2pt}
\paragraph{Priors on entities.}
Additionally, in practice, while one should not ``judge a book by its cover'',
a Bayesian would have a prior on the book's worth based on metadata, or on other users' ratings.
This could allow them to optimize their active learning, 
thereby easily rating negatively entities that are extremely likely to be bad,
while carefully comparing those that look like they might be top  entities.

%

\section{Conclusion}
\label{sec:conclusion}
This paper introduced SCoRa, a new model that learns cardinal preferences
out of both direct ratings and potentially quantitative comparisons.
Our new model leverages the flexibility of the GBT models, 
to provide numerous desirable theoretical guarantees.
Our basic evaluation also shows good resilience to model mismatch.
Perhaps more interestingly, 
we used our SCoRa model to provide new insights into the classical debate
over whether ratings or comparisons should be favored.
While there are clearly settings where one strictly outperforms the other, 
assuming that comparisons are more costly,
we exhibited a setting, which we believe to be realistic,
where combining both comparisons and ratings is best.
We believe that our findings could help inform how to best elicit human preferences
to improve the performance and/or the ethics of algorithms.




\newpage
\section*{Impact Statement}

This paper presents work on advancing preference learning. 
While we acknowledge that our findings may be used to increase the addictiveness
of AI-driven products like chatbots and social networks,
we also believe they are essential to construct more trustworthy AIs,
whose construction better matches democratic norms
and whose behavior better reflects citizens' preferred AI behaviors.




\bibliographystyle{icml2026}
\bibliography{references}


\newpage
\appendix
\onecolumn
\onecolumn

\section{Reduction of SCoRa Models to Flexible GBT}
\label{app:SCoRa=flex}

\subsection{More precise formalization}

Let us first better formalize Theorem~\ref{theo:SACEisFlex}.
When we argue that any SCoRa model can be framed as a flexible GBT model,
we mean that there are three functions $H_1, H_2, H_3$ such that
\begin{itemize}
    \item $H_1$ maps any tuple of SCoRa hyperparameters $(D, \sigma_\beta^2, \sigma_0^2, f, g, x)$ to a tuple of flexible GBT hyperparameters $(\tilde D, \Sigma_\beta, (f_{ab}), \tilde x)$,
    \item $H_2$ maps any SCoRa latent vector $(\beta, \theta_0)$ to a flexible GBT latent vector $\tilde \beta$.
    \item $H_3$ maps injectively any SCoRa dataset $\bD$ to a flexible GBT dataset $\tilde \bD$,
    This function may also leverage SCoRa hyperparameters.
\end{itemize}
Moreover, we demand that, intuitively, $H_3$ commute with $(H_1, H_2)$.
More precisely, we want the distribution of $H_3(\bD)$ 
given hyperparameters $(D, \sigma_\beta^2, \sigma_0^2, f, g, x)$ 
and latent variables $(\beta, \theta_0)$
to coincide with the distribution of $\tilde \bD$ 
given hyperparameters $H_1(D, \sigma_\beta^2, \sigma_0^2, f, g, x)$ 
and $H_2(\beta, \theta_0)$.

This formalizes the theorem.
Indeed, if the distributions do coincide,
then $\tilde \bD$ is, with probability 1, in the image of $H_3$.
Since $H_3$ is invertible, this means that we can reconstruct the data
$\bD$ as $H_3^{-1}(\tilde \bD)$.
Thereby we have generated the data $\bD$,
by leveraging the flexible GBT.

Perhaps most importantly for inference, 
assuming additionally that $H_2$ be injective.
this construction also guarantees 
that the posterior distribution on SCoRa latent variables given $\bD$
is the distribution of $H_2^{-1}(\tilde \beta)$ given $H_3(\bD)$.
As we especially focus on MAP, 
this will imply that we can reconstruct guarantees on 
$(\beta^*, \theta_0^*) = H_2^{-1}(\tilde \beta^*)$ 
out of guarantees on $\tilde \beta^*$.

\subsection{Proof of Theorem \ref{theo:SACEisFlex}}

\begin{proof}[Proof of Theorem \ref{theo:SACEisFlex}]
Consider a SCoRa model with hyperparameters 
$(D, \sigma_\beta^2, \sigma_0^2, f, g, x)$
and data $\bA \cup \bC$.
We construct the corresponding flexible GBT model 
by mapping the latent vector $\beta \in \setR^D$ and its prior variance $\sigma_\beta^2$ to 
\begin{equation}
    \tilde{\beta} = H_2(\beta, \theta) \triangleq \begin{pmatrix} \beta \\ \theta_0 \end{pmatrix} \in \setR^{D+1}, 
    \qquad
    \tilde \Sigma_\beta \triangleq \textsc{Diag} \left(
        \underbrace{\sigma_\beta^2, \ldots, \sigma_\beta^2}_{D~\text{values}},
        \sigma_0^2
    \right) \in \setR^{(D+1) \times (D+1)}.
\end{equation}
Moreover, we extend the embeddings $x \in \setR^{D \times A}$ of the entities
by defining the flexible GBT embeddings
\begin{equation}
\tilde{x}_a \triangleq \begin{pmatrix} x_a \\ 0 \end{pmatrix}, 
\qquad
\tilde{x}_{A + 1} \triangleq \begin{pmatrix} 0_D \\ 1 \end{pmatrix},
\end{equation}
thereby yielding the embedding matrix $\tilde x \in \setR^{(D+1) \times (A+1)}$.
Put differently, we defined the hyperparameter reparametrization
$H_1(D, \sigma_\beta^2, \sigma_0^2, f, g, x) \triangleq (\tilde D, \tilde \Sigma_\beta, \tilde x)$.

We then construct a dataset $\tilde \bD \triangleq H_3(\bD)$ 
for flexible GBT out of the dataset $\bA \cup \bC$ for SCoRa
(and implicitly using the hyperparameters of SCoRa).
We first map each SCoRa comparison $(n, b_n, c_n, r_n)$ for $n \in [\Nc]$
to the flexible-GBT comparison $(n, b_n, c_n, r_n, f)$.
We then map each SCoRa rating $(n, a_n, t_n)$ for $n \in [\Na]$
to the flexible-GBT comparison $(\Nc + n, a_n, A+1, t_n, g)$.
Collecting all these data then yields $\tilde \bD$.

We then prove the equivalence of the original SCoRa model
and of the constructed flexible GBT model
by showing that their negative log-likelihood functions match.
Note that for $a \in [A]$, 
$\tilde x_a^\top \tilde \beta = x_a^\top \beta$
and $\tilde x_{A+1}^\top \tilde \beta = \theta_0$.
Indeed, we have
\begin{align}
    &\mathcal L_{FGBT} (\tilde \beta | \tilde \bD)
    \triangleq 
        \frac{\tilde \beta \tilde \Sigma_\beta^{-1} \tilde \beta}{2} 
        + \sum_{(\tilde a, \tilde b, \tilde r, \tilde f)} \left(
            \Phi_{\tilde f} (\tilde x_{\tilde a \tilde b}^\top \tilde \beta)SAC
            - \tilde r \tilde x_{\tilde a \tilde b}^\top \tilde \beta
        \right) \\
    &= 
        \frac{1}{2} \begin{pmatrix} \beta \\ -\theta_0 \end{pmatrix}^\top 
            \textsc{Diag} \left(
                \underbrace{\sigma_\beta^2, \ldots, \sigma_\beta^2}_{D~\text{values}},
                \sigma_0^2
            \right)^{-1}
            \begin{pmatrix} \beta \\ -\theta_0 \end{pmatrix} 
         + \sum_{n=1}^{\Nc} \left(
            \Phi_{f} (\tilde x_{b_n c_n}^\top \tilde \beta)
            - r_n \tilde x_{b_n c_n}^\top \tilde \beta
        \right) \notag \\
        &\qquad \qquad \qquad \qquad \qquad \qquad \qquad \qquad \qquad
        + \sum_{n=1}^{\Nc}\left(
            \Phi_{g} (\tilde x_{a_n}^\top \tilde \beta)
            - t_n \tilde x_{a_n}^\top \tilde \beta
        \right) \\
    &= 
        \frac{1}{2} \begin{pmatrix} \beta \\ -\theta_0 \end{pmatrix}^\top 
            \textsc{Diag} \left(
                \underbrace{\sigma_\beta^{-2}, \ldots, \sigma_\beta^{-2}}_{D~\text{values}},
                \sigma_0^{-2}
            \right)
            \begin{pmatrix} \beta \\ -\theta_0 \end{pmatrix} 
        + \sum_{(b, c, r) \in \bC} \left(
            \Phi_{f} (\tilde x_{b c}^\top \tilde \beta) - r \tilde x_{b c}^\top \tilde \beta
        \right) \notag \\
        &\qquad \qquad \qquad \qquad \qquad \qquad \qquad \qquad \qquad
        + \sum_{(a, t) \in \bA}\left(
            \Phi_{g} (\tilde x_{a}^\top \tilde \beta - \tilde x_{A+1}^\top \tilde \beta)
            - t_n (\tilde x_{a}^\top \tilde \beta - x_{A+1}^\top \tilde \beta)
        \right) \\
    &= \frac{\norm{\beta}{2}^2}{2\sigma_\beta^2} + \frac{\theta_0^2}{2\sigma_0^2}
        + \sum_{(b, c, r) \in \bC} \left(
            \Phi_{f} (x_{b c}^\top \beta) - r x_{b c}^\top \beta
        \right)
        + \sum_{(a, t) \in \bA}\left(
            \Phi_{g} (x_{a}^\top \beta - \theta_0)
            - t_n (x_{a}^\top \beta - \theta_0)
        \right) \\
    &= \mathcal L_{SCoRa} (\beta | \bA \cup \bC)
\end{align}
This concludes the proof.
\end{proof}
    
\section{Proofs of Section \ref{sec:model}}
\label{app:theo}

\subsection{Mean Relation for SCoRa Model (Proposition ~\ref{prop:MAPthreshold})}
\label{proof:propmean}

Before proving Proposition~\ref{prop:MAPthreshold},
we prove a lemma on flexible GBT.

\begin{lemma}
\label{lemma:fgbt_no_embedding}
    Assume $x = I_A$ in flexible GBT (i.e. no embedding).
    Then for any flexible-GBT dataset $\bD$,
    we have ${\bf 1}^T \Sigma^{-1}_\beta \beta^*(\bD) = 0$.
\end{lemma}

\begin{proof}[Proof of Lemma~\ref{lemma:fgbt_no_embedding}]
    Denote $\beta^* \triangleq \beta^*(\bD)$ the unique MAP of flexible GBT
    with dataset $\bD$.
    Now define $h(t) \triangleq \mathcal L(\beta^* + t {\bf 1} | \bD) - \mathcal L(\beta^* | \bD)$,
    where ${\bf 1} \triangleq (1, 1, \ldots, 1) \in \setR^D$.
    Note that $h$ is clearly differentiable.
    Given that $x_{ab}^T {\bf 1} = x_a^T {\bf 1} - x_b^T {\bf 1} = 1 - 1 = 0$,
    we have $x_{ab}^T (\beta + t {\bf 1}) = x_{ab}^T \beta$.
    We then have
    \begin{align}
        h(t) 
        &= \frac{(\beta + t {\bf 1})^T \Sigma_\beta^{-1} (\beta + t {\bf 1})}{2} 
            + \sum_{(a, b, r, f) \in \bD} \left( 
                \Phi_f(x_{ab}^T (\beta + t {\bf 1})) - r x_{ab}^T (\beta + t {\bf 1}) 
            \right) \notag \\
        &\qquad \qquad \qquad - \frac{\beta^T \Sigma_\beta^{-1} \beta}{2} 
            - \sum_{(a, b, r, f) \in \bD} \left( \Phi_f(x_{ab}^T \beta) - r x_{ab}^T \beta \right) \\
        &= \frac{(\beta + t {\bf 1})^T \Sigma_\beta^{-1} (\beta + t {\bf 1})}{2} 
        - \frac{\beta^T \Sigma_\beta^{-1} \beta}{2}.
    \end{align}
    It follows that
    \begin{align}
        h'(t) = {\bf 1}^T \Sigma_\beta^{-1} (\beta + t {\bf 1}).
    \end{align}
    By optimality of $\beta^*$, 
    we know that $0 = h'(0) = {\bf 1}^T \Sigma_\beta^{-1} \beta$,
    which is the lemma.
\end{proof}

We can now complete the proof of Proposition~\ref{prop:MAPthreshold}.

\begin{proof}[Proof of Proposition~\ref{prop:MAPthreshold}]
Applying Lemma~\ref{lemma:fgbt_no_embedding} to the corresponding SCoRa model
implies 
\begin{align}
    {\bf 1}_{D+1} \textsc{Diag} \left( \sigma_\beta^{-2}, \ldots, \sigma_\beta^{-2}, \sigma_0^{-2} \right) \begin{pmatrix}
        \beta^*(\bD) \\ \theta_0^*(\bD)
    \end{pmatrix} = 0.
\end{align}
Expanding this equation yields
\begin{align}
    \sigma_\beta^{-2} \sum_{d = 1}^D \beta^*_d
    + \sigma_0^{-2} \theta_0^* = 0.
\end{align}
Now, given that $x = I_D$ (with $A = D$), 
we have $\theta_d^* = x_d^T \beta^* = \beta^*_d$.
Substituting symbolically $d$ with $a$ then yields
\begin{align}
    \sigma_\beta^{-2} \sum_{a = 1}^A \theta^*_a
    + \sigma_0^{-2} \theta_0^* = 0.
\end{align}
Rearranging the terms yields the proposition.
\end{proof}

\subsection{Conditional Moments (Proposition \ref{prop:moments_embedding})}
\label{app:moments_flexGBT}

We demonstrate Proposition~\ref{prop:moments_embedding} by considering the flexible linear GBT model with embedding $x$.

\begin{lemma}[Conditional moments]
\label{lemma:moments}
    Under flexible GBT, upon comparing $a$ to $b$ under root law $f$, 
    we have
    \begin{align}
        \expect [r | \beta] = \Phi_f' (x_{ab}^\top \beta)
        \qquad \text{and} \qquad
        \mathbb V [r | \beta] = \Phi_f'' (x_{ab}^\top \beta).
    \end{align}
\end{lemma}


\begin{proof}
Define the conditional cumulant generating function (CGF) of $r$ given $\beta$ as
\begin{equation}
K_{r}(\lambda) \triangleq \log \mathbb{E}\big[\mathrm{e}^{\lambda r} \mid \beta\big]
= \log \int_{\mathbb{R}} \mathrm{e}^{\lambda r} \, p(r \mid \beta) \, \mathrm{d}r.
\end{equation}
Substituting the likelihood yields
\begin{align}
K_{r}(\lambda) 
&= \log \int_{\setR} e^{\lambda r} \exp(r \, x_{ab}^\top \beta - \Phi_f(x_{ab}^\top \beta)) \mathrm{d} F(r) \\
&= \log \left( e^{- \Phi_f(x_{ab}^\top \beta)}  \int_{\setR} \exp \left(r (\lambda + x_{ab}^\top \beta) \right) \mathrm{d} F(r) \right) \\
&= - \Phi_f (x_{ab}^\top \beta) + \log \int_{\mathbb{R}} \exp(r (\lambda + x_{ab}^\top \beta)) \, \mathrm{d} F(r)  \\
&= \Phi_f (\lambda + x_{ab}^\top \beta) - \Phi_{f}(x_{ab}^\top \beta),
\end{align}
By definition of a cumulant generating function, 
the first derivative at $\lambda=0$ gives the mean
\begin{equation}
\mathbb{E}[r \mid \beta] = K_{r}'(0) 
= \frac{d}{d\lambda} \Big[ \Phi_{f}(\lambda + x_{a b}^\top \beta) - \Phi_{a b}(x_{a b}^\top \beta) \Big]_{\lambda=0} 
= \Phi_{f}'(x_{a b}^\top \beta).
\end{equation}
It is also well-known that the second derivative at $\lambda=0$ gives the variance:
\begin{equation}
\mathbb{V}[r \mid \beta] = K_{r}''(0) = \frac{d^2}{d\lambda^2} \Big[ \Phi_{f}(\lambda + x_{ab}^\top \beta) - \Phi_{f}(x_{a b}^\top \beta) \Big]_{\lambda=0} 
= \Phi_{f}''(x_{a b}^\top \beta).
\end{equation}
This proves the lemma.
\end{proof}

\begin{proof}[Proof of Proposition~\ref{prop:moments_embedding}]
Recall that, by Theorem~\ref{theo:SACEisFlex}, 
each SCoRa observation can be written as a flexible GBT observation,
where SCoRa hyperparameters, latent variables and observables are mapped (injectively)
to flexible-GBT hyperparameters, latent variables and observables.
Typically $x$ is mapped to $\tilde x$, $\beta$ to $\tilde \beta$ and $\bD$ to $\tilde \bD$.

\paragraph{Comparisons.}
More specifically, consider a SCoRa comparison $(b, c, r)$.
Its corresponding flexible-GBT comparison is $(b, c, r, f_{bc})$.
Moreover recall that we proved that 
$\tilde{x}_{bc}^\top\tilde{\beta}=x_{bc}^\top\beta$.
Its conditional density therefore has the flexible-GBT form
\begin{equation}
p(r\mid\tilde{\beta})
= f_{bc}(r)
  \exp\!\big(r\,x_{bc}^\top\beta
            -\Phi_{bc}(x_{bc}^\top\beta)\big).
\end{equation}
Thus the conditional law of \(r_n\) given \(\beta\) is the same as in the flexible GBT model with embedding \(x\).  
Lemma~\ref{lemma:moments} then gives
\begin{equation}
\mathbb{E}[r\mid\beta]=\Phi_{bc}'(x_{bc}^\top\beta),
\qquad
\mathbb{V}[r\mid\beta]=\Phi_{bc}''(x_{bc}^\top\beta).
\end{equation}

\paragraph{Ratings.}
A SCoRa rating $(a, t)$ is mapped to a flexible-GBT comparison $(a, A+1, t, g)$.
Moreover, we proved $\tilde{x}_{a0}^\top\tilde{\beta} = x_a^\top\beta-\theta_0$.
The conditional density is therefore
\begin{equation}
p(t\mid\tilde{\beta})
= f_{a0}(t)\,
  \exp\!\big(t(x_{a}^\top\beta-\theta_0)
            -\Phi_{g}(x_{a}^\top\beta-\theta_0)\big),
\end{equation}
which matches the flexible-GBT form with natural parameter 
\(x_{a}^\top\beta-\theta_0\).  
Applying Lemma~\ref{lemma:moments} gives
\begin{equation}
\mathbb{E}[t\mid\beta,\theta_0]
   =\Phi_{g}'(x_{a}^\top\beta-\theta_0),\qquad
\mathbb{V}[t\mid\beta,\theta_0]
   =\Phi_{g}''(x_{a}^\top\beta-\theta_0).
\end{equation}
These are precisely the identities stated in
\eqref{eq:momentr_embed}--\eqref{eq:momentt_embed}.
\end{proof}

\subsection{Pairwise Monotonicity (Proposition \ref{prop:monotonicity_SCoRa_embed})}
\label{app:pairwise_monotone}

We first formulate and prove pairwise monotonicity 
for the general flexible linear GBT model with embedding matrix $x$.

\begin{lemma}[Flexible GBT is pairwise monotonous]
\label{lemma:pairwise_monotone_flexGBT}
For any two datasets $\bD$ and $\bD'$ that differ only on the $n$-th comparison
with $a_n = a_n'$, $b_n = b_n'$, $f_n = f_n'$ and $r_n \geq r_n'$,
we have $\theta^*_{a_n b_n} (\bD) \geq \theta^*_{a_n b_n} (\bD')$.
\end{lemma}

\begin{proof}
    To ease notation, 
    denote $a \triangleq a_n$, $b  \triangleq b_n$, $f \triangleq f_n'$.
    For $r \in [r_n', r_n]$, 
    denote $\bD(r)$ the dataset obtained by taking $\bD$,
    and replacing the $n$-th comparison with $(a, b, r, f)$.
    We then have
    \begin{align}
        \mathcal L(\beta | \bD(r))
        &= \mathcal L(\beta | \bD) 
            - \left( 
                \Phi_f(x_{ab}^\top \beta) - r_n x_{ab}^\top \beta
            \right)
            + \Phi_f(x_{ab}^\top \beta) - r x_{ab}^\top \beta \\
        &= \mathcal L(\beta | \bD) + (r_n - r) x_{ab}^\top \beta.
    \end{align}
    Let $\beta^*(r) \triangleq \beta^*(\bD(r))$.
    By the optimality condition, we know that
    \begin{align}
        0 
        &= \nabla_{|\beta = \beta^*(r)} \mathcal L(\beta | \bD(r))
        = \nabla_{|\beta = \beta^*(r)} \left[ 
            \mathcal L(\beta | \bD) 
            + (r_n - r) x_{ab}^\top \beta
        \right] \\
        &= \nabla \mathcal L(\beta^*(r) | \bD) + (r_n - r) x_{ab}.
    \end{align}
    The right-hand side is smooth in both $r$.
    Differentiating by $r$ then yields
    \begin{equation}
        0 
        = \nabla^2 \mathcal L(\beta^*| \bD) \frac{d\beta^*}{dr} - x_{ab}.
    \end{equation}
    We previously proved that $\mathcal L$ was strongly convex,
    thus the Hessian matrix is definite positive,
    and hence invertible.
    Rearranging terms then yields
    \begin{align}
        \frac{d\beta^*}{dr} = \nabla^2 \mathcal L(\beta^*| \bD)^{-1} x_{ab}.
    \end{align}
    Finally, we have
    \begin{equation}
        \frac{d\theta_{ab}^*(\bD(r))}{dr} 
        = \frac{d x_{ab}^\top \beta^*(r)}{dr}
        = x_{ab}^\top \frac{d\beta^*}{dr}
        = x_{ab}^\top \nabla^2 \mathcal L(\beta^*| \bD)^{-1} x_{ab} \geq 0,
    \end{equation}
    using the fact that $\nabla^2 \mathcal L$  is definite positive.
    This concludes the proof.
\end{proof}

We can now derive the SCoRa case.

\begin{proof}[Proof of Proposition~\ref{prop:monotonicity_SCoRa_embed}]
    We note that for entities $a \in [A]$,
    the corresponding flexible-GBT MAP scores match the SCoRa MAP scores, as
    \begin{align}
        \tilde \theta^*_{a}
        \triangleq \tilde x_a^\top \tilde \beta^*
        = x_a^\top \beta^*
        = \theta^*_a.
    \end{align}
    Similarly, we have $\tilde \theta^*_{A+1} = \theta^*_0$.
    Combining these two equalities to Lemma~\ref{lemma:pairwise_monotone_flexGBT} 
    then implies pairwise monotonicity for SCoRa.
\end{proof}


\subsection{Update Properties (Proposition \ref{prop:direction_comparison})}
\label{app:update_flexGBT}

\begin{lemma}[Zero update in flexible GBT models]
\label{lemma:zero_update}
Consider any flexible GBT dataset $\bD$, 
and denote $\bD'$ the dataset obtained by adding the comparison $(a, b, r, f)$ to $\bD$.
We have the implication
\begin{align}
r = \Phi_f' \left(\theta_{ab}^*(\bD)\right)
\implies
\theta_{ab}^*(\bD') - \theta_{ab}^*(\bD).
\end{align}
\end{lemma}

\begin{proof}
    We have
    \begin{align}
        \mathcal L(\beta | \bD')
        = \mathcal L(\beta | \bD) 
            + \Phi_f (x_{ab}^\top \beta) - r x_{ab}^\top \beta.
    \end{align}
    Differentiating by $\beta$ yields
    \begin{align}
        \nabla \mathcal L(\beta | \bD')
        = \nabla \mathcal L(\beta | \bD) 
            + \underbrace{\left( \Phi_f' (x_{ab}^\top \beta) - r \right)}_{=0} x_{ab}
        = \nabla \mathcal L(\beta | \bD) .
    \end{align}
    Plugging $\beta^*(\bD)$ in the equation and using
    $\nabla \mathcal L(\beta^*(\bD) | \bD) = 0$
    then yields $\nabla \mathcal L(\beta^*(\bD) | \bD') = 0$.
    By strong convexity of $\mathcal L$,
    this implies that $\beta^*(\bD)$ is the MAP for $\bD'$,
    hence the lemma.
\end{proof}

\begin{lemma}[Update properties in flexible GBT models]
\label{lemma:updates_flexGBT}
Consider any flexible GBT dataset $\bD$, 
and denote $\bD'$ the dataset obtained by adding the comparison $(a, b, r, f)$ to $\bD$.
Then
\begin{align*}
\Sign \left( \theta_{ab}^*(\bD') - \theta_{ab}^*(\bD) \right)
=
\Sign \left( r - \Phi_f' \left(\theta_{ab}^*(\bD)\right) \right).
\end{align*}

\end{lemma}

\begin{proof}
    This follows from combining the zero update property 
    (Lemma~\ref{lemma:zero_update})
    with pairwise monotonicity (Lemma~\ref{lemma:pairwise_monotone_flexGBT}).
\end{proof}

We can now derive the SCoRa case.

\begin{proof}[Proof of Proposition~\ref{prop:direction_comparison}]
    Again we use the fact that SCoRa MAP scores 
    match their corresponding flexible GBT MAP scores.
\end{proof}

\subsection{Lipschitz resilience (Proposition \ref{prop:resilience_SCoRa})}
\label{app:resilience_flexGBT_pairs}

Given two flexible GBT datasets $\bD, \bD'$,
we define their distance $\Delta(\bD, \bD')$ 
as the minimal number of elementary edits
needed to transform $\bD$ into $\bD'$.
An elementary edit may be the removal, the addition 
or the modification of a comparison.
Moreover, for any root law $f$ (assumed to be an even function), 
we define $\textsc{Supp}(f)$ the support of $f$,
which corresponds to the legitimate comparisons $r$ that can be submitted.
Finally, for any matrix $\Sigma \in \setR^{D \times D}$, 
define the operator norm
\begin{align}
    \norm{\Sigma}{op} 
    \triangleq \sup_{x \in \setR^D - \set{0}} \frac{\norm{\Sigma x}{2}}{\norm{x}{2}}.
\end{align}
It is well-known that for definite-positive matrix,
this norm equals the largest eigenvalue of the matrix.
We can now formulate Lipschitz resilience for flexible GBT.

\begin{lemma}
\label{lemma:resilience_fgbt}
    Consider flexible GBT with an embedding matrix $x \in \setR^{D \times A}$,
    a prior covariance matrix $\Sigma_\beta \in \setR^{D \times D}$
    and a set $\mathcal F$ of root laws.
    Then, for any two flexible-GBT datasets $\bD, \bD'$,
    \begin{align}
        \norm{\beta^*(\bD) - \beta^*(\bD')}{2}
        &\leq \Delta(\bD, \bD') 
            \norm{\Sigma_\beta}{op} 
            \norm{x}{2, \infty}
            \sup_{f \in \mathcal F} \textsc{Supp} (f) \\
        \norm{\theta^*(\bD) - \theta^*(\bD')}{2}
        &\leq \Delta(\bD, \bD') 
            \norm{\Sigma_\beta}{op} 
            \norm{x}{2, \infty}
            \norm{x}{2}
            \sup_{f \in \mathcal F} \textsc{Supp} (f).
    \end{align}
\end{lemma}

\begin{proof}
    To move from $\bD$ to $\bD'$, 
    denote $\bD^-$ the comparisons that had to be removed,
    $\bD^+$ those that had to be added
    and $\bD^{\neq}$ those that were modified.
    We then have
    \begin{align}
        \mathcal L(\beta|\bD)
        &= \mathcal L(\beta|\bD')
        + \sum_{(a, b, r, f) \in \bD^+} \left(
            \Phi_f(x_{ab}^\top \beta) - r x_{ab}^\top \beta
        \right) \notag \\
        &\qquad - \sum_{(a, b, r, f) \in \bD^-} \left(
            \Phi_f(x_{ab}^\top \beta) - r x_{ab}^\top \beta
        \right)
        + \sum_{(a, b, r, r', f) \in \bD^{\neq}} \left(
            r' - r
        \right) x_{ab}^\top \beta.
    \end{align}
    Rearranging terms and differentiating by $\beta$ yields, for all $\beta$,
    \begin{align}
        &\norm{
            \nabla \mathcal L(\beta|\bD) 
            - \nabla \mathcal L(\beta|\bD')
        }{2}
        \leq
        \sum_{(a, b, r, f) \in \bD^+ \cup \bD^-} \norm{
            \left( \Phi_f(x_{ab}^\top \beta) - r  \right) x_{ab}
        }{2} 
        + \sum_{(a, b, r, r', f) \in \bD^{\neq}} \norm{
            (r' - r) x_{ab}
            }{2} \\
        &\leq
        \sum_{(a, b, r, f) \in \bD^+ \cup \bD^-} 
            \absv{ \Phi_f(x_{ab}^\top \beta) - r } 
            \norm{x_{ab}}{2}
        + \sum_{(a, b, r, r', f) \in \bD^{\neq}} 
            \absv{r' - r} \norm{x_{ab}}{2} \\
        &\leq
        \sum_{(a, b, r, f) \in \bD^+ \cup \bD^-} 
            \left( \absv{ \Phi_f(x_{ab}^\top \beta)} + \absv{r} \right)
            \left( \norm{x_a}{2} + \norm{x_b}{2} \right)
        + \sum_{(a, b, r, r', f) \in \bD^{\neq}} 
            \absv{r' - r} \norm{x_{ab}}{2}.
    \end{align}
    For any comparison $(a, b, r, f)$,
    note that $\absv{ \Phi_f(x_{ab}^\top \beta) } \leq \textsc{Supp}(f)$
    and $\absv{ r } \leq \textsc{Supp}(f)$.
    Similarly, $\absv{ r' } \leq \textsc{Supp}(f)$.
    Bounding also $\norm{x_a}{2}$ and $\norm{x_b}{2}$
    by $\norm{x}{2, \infty}$, we then have
    \begin{align}
        \forall \beta \mathsep
        \norm{
            \nabla \mathcal L(\beta|\bD) 
            - \nabla \mathcal L(\beta|\bD')
        }{2}
        &\leq \card{\bD^+ \cup \bD^-} 
            2 \set{ \sup_{f \in \mathcal F} \textsc{Supp}(f) }
            \cdot 2 \norm{x}{2, \infty} \\
        &\leq 4 \Delta(\bD, \bD') \norm{x}{2, \infty}
            \sup_{f \in \mathcal F} \textsc{Supp}(f).
    \end{align}
    Plugging $\beta^*(\bD')$ then yields
    \begin{equation}
        \norm{\nabla \mathcal L(\beta^*(\bD') | \bD)}{2}
        \leq 4 \Delta(\bD, \bD') \norm{x}{2, \infty}
            \sup_{f \in \mathcal F} \textsc{Supp}(f).
    \end{equation}
    We now leverage the strong convexity of $\mathcal L$
    with parameter $\lambda_{min}(\Sigma_\beta^{-1}) = \norm{\Sigma_\beta}{op}^{-1}$,
    we also have
    \begin{align}
        \norm{\beta^*(\bD) - \beta^*(\bD')}{2}
        &\leq \norm{\Sigma_\beta}{op} \norm{
            \underbrace{\nabla L(\beta^*(\bD) | \bD)}_{=0}
            - \nabla \mathcal L(\beta^*(\bD') | \bD)
        }{2} \\
        &= \norm{\Sigma_\beta}{op} \norm{\nabla \mathcal L(\beta^*(\bD') | \bD)}{2} \\
        &\leq 4 \Delta(\bD, \bD') \norm{\Sigma_\beta}{op} \norm{x}{2, \infty}
            \sup_{f \in \mathcal F} \textsc{Supp}(f).
    \end{align}
    To derive the case for $\theta^*$,
    we note that $\theta^* = x^\top \beta^*$, hence
    \begin{align}
        \norm{\theta^*(\bD) - \theta^*(\bD')}{2}
        &= \norm{x^\top (\beta^*(\bD) - \beta^*(\bD'))}{2}
        \leq \norm{x}{2} \norm{\beta^*(\bD) - \beta^*(\bD')}{2} \\
        &\leq 4 \Delta(\bD, \bD') \norm{\Sigma_\beta}{op} 
            \norm{x}{2} \norm{x}{2, \infty}
            \sup_{f \in \mathcal F} \textsc{Supp}(f).
    \end{align}
    This concludes the proof.
\end{proof}

\begin{proof}[Proof of Proposition \ref{prop:resilience_SCoRa}]

By applying Lemma~\ref{lemma:resilience_fgbt} for this flexible GBT model
constructed out of a SCoRa model, 
we have the Lipschitz bound on \(\tilde{\beta}^*\):
\begin{align}
\|\tilde{\beta}^*(\tilde \bD) - \tilde{\beta}^*(\bD')\|_2
\leq 4 \sigma_{\max}^2 \max \set{R_{\max}, T_{\max}} \|\tilde{x}\|_{2,\infty}
    \Delta\big(\bD, \bD'\big),
\end{align}
where, due to $\|x_0\| = 1$,
\begin{align}
\|\tilde{x}\|_{2,\infty} 
=  \max_{a \in [A]_0} \|\tilde{x}_a\|_2 
= \max \set{ \max_{a\in [A]} \|x_a\|_2, 1 }
= \max \set{ \|{x}\|_{2,\infty} , 1 }.
\end{align}
Since \(\tilde{\beta}^* = (\beta^*, \theta_0^*)\), 
the triangle inequality yields separate bounds:
\begin{align}
\|\beta^*(\bD) - \beta^*(\bD')\|_2 
\leq 4 \sigma_{\max}^2 \max \set{R_{\max}, T_{\max}} 
    \max \set{ \|{x}\|_{2,\infty} , 1 } \, 
    \Delta\big(\bD, \bD'\big),
\end{align}
and
\begin{align}
\label{eq:part1resilience}
|\theta_0^*(\bD) - \theta_0^*(\bD')| 
\leq 4 \sigma_{\max}^2 \max \set{R_{\max}, T_{\max}} 
    \max \set{ \|{x}\|_{2,\infty} , 1 } \, 
    \Delta\big(\bD, \bD'\big).    
\end{align}
Finally, since the scores satisfy $\theta^* = x^\top \beta^*$,
we have  
\begin{align} \label{eq:part2resilience}
\|\theta^*(\bD) - \theta^*(\bD')\|_2 
\leq \|x\|_{2} \|\beta^*(\bD) - \beta^*(\bD')\|_2,
\end{align}
and therefore
\begin{align}
\|\theta^*(\bD) - \theta^*(\bD')\|_2 
\leq 4 \sigma_{\max}^2 
    \max \set{R_{\max}, T_{\max}} \|x\|_{2}
    \max \set{ \|{x}\|_{2,\infty} , 1 } \, 
    \Delta\big(\bD, \bD'\big).
\end{align}
Finally,the relations \eqref{eq:thetatheta0Lipschitz} follows exactly from \eqref{eq:part1resilience} and \eqref{eq:part2resilience}.
\end{proof}

\section{Additional Experiments}
\label{app:experiments}

We present an additional experiment to compare active learning when the first phase is done with ratings (as in Section \ref{sec:SCoRa_value}) to active learning when the first phase is done with comparisons.
The structure of the experiment is very similar to the ones of Section \ref{sec:SCoRa_value}.
We have a fixed budget of $\budget=10000$, which we can allocate to comparisons or ratings, or a mix of the two. 
In addition to the setting of Section \ref{sec:SCoRa_value}, we compare the performance when using comparisons instead of ratings for the first phase of the training.
In other words, the $\pc$ fraction of the budget is allocated to comparisons as in Section \ref{sec:SCoRa_value},
but the first estimate of the scores (which is done using the $1-\pc$ fraction of the budget) is this time also done with comparisons.
We use $5$-ary ratings and binary comparisons.
The results are presented in Figure \ref{fig:baseline}.

\begin{figure*}[ht]
\centering
    \includegraphics[width=0.49\linewidth]{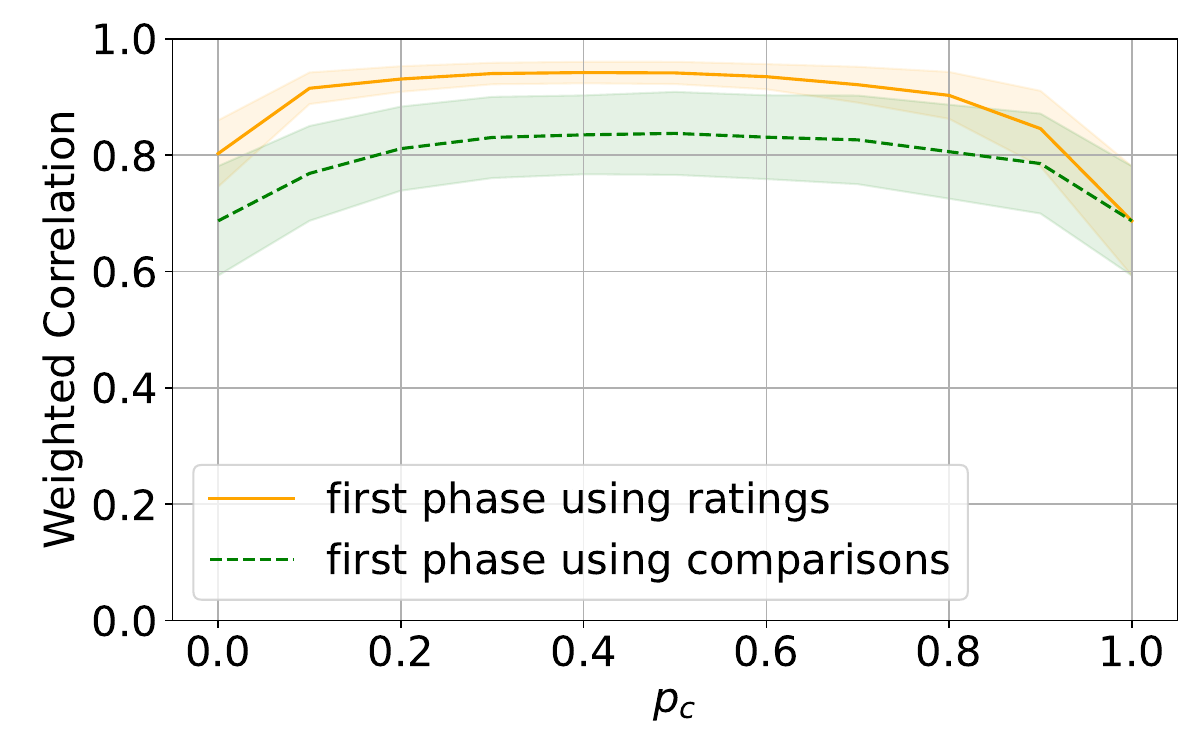}   
    \caption{Comparison of active learning when the first phase is done with ratings to active learning when the first phase is done with comparisons.
    We use parameters $k_r=5$, $k_c=2$, $\budget=10000$, $c_r=1$, $c_c=3$ and one-hot-encoded embeddings.}
    \label{fig:baseline}
\end{figure*}

\paragraph{Results}
We can see on Figure \ref{fig:baseline} that using ratings for the first part of the learning achieves the overall best recovery of the true scores, for a value of $\pc$ around $0.5$.
We observe the same plateau effect as in Section \ref{sec:SCoRa_value}, with fast improvement when getting away from extreme values of $p_c$, although not as strong as Figure \ref{fig:sweet-one-hot}.
This confirms the results of Section \ref{sec:SCoRa_value} by showing that mixing comparisons and ratings for active learning is the best option in this setting, also beating an active learning using comparisons for the first phase.


\end{document}